\documentclass[letterpaper, 10 pt, conference]{ieeeconf}  

\usepackage{amsmath,amssymb,amsfonts}
\usepackage{physics}
\usepackage{amssymb}
\usepackage{algorithmic}
\usepackage{graphicx}
\usepackage{textcomp}
\usepackage{mathrsfs}

\usepackage{booktabs} 
\usepackage{tikz} 
\usepackage{mathtools}
\DeclareMathOperator*{\argmax}{arg\max}
\DeclareMathOperator*{\argmin}{arg\min}

\usepackage{amsfonts}
\newtheorem{theorem}{Theorem}[section]
\newtheorem{proposition}{Proposition}
\newtheorem{lemma}{Lemma}

\newtheorem{definition}{Definition}

\newtheorem{assumption}{Assumption}

\usepackage{hyperref}
\usepackage{multirow, tabu}
\usepackage{url}
\usepackage{amssymb}

\IEEEoverridecommandlockouts                              

\title{\LARGE \bf
Optimization Solution Functions as Deterministic Policies for Offline Reinforcement Learning
}

\author{Vanshaj Khattar$^{1}$ and Ming Jin$^{1}$
\thanks{*This work is supported by NSF \#2331775, the Commonwealth Cyber Initiative (CCI), and C3.ai Digital Transformation Institute.}
 \thanks{$^{1}$Virginia Tech, USA.
        {\tt\small \{vanshajk, jinming\}@vt.edu }}
}

\begin{document}

\maketitle
\thispagestyle{empty}
\pagestyle{empty}

\begin{abstract}
Offline reinforcement learning (RL) is a promising approach for many control applications but faces challenges such as limited data coverage and value function overestimation. In this paper, we propose an implicit actor-critic (iAC) framework that employs optimization solution functions as a deterministic policy (actor) and a monotone function over the optimal value of optimization as a critic. By encoding “optimality” in the actor policy, we show that the learned policies are robust to the suboptimality of the learned actor parameters via the exponentially decaying sensitivity (EDS) property. 
We obtain performance guarantees for the proposed iAC framework and show its benefits over general function approximation schemes. Finally, we validate the proposed framework on two real-world applications and show a significant improvement over state-of-the-art (SOTA) offline RL methods.
\end{abstract}
\section{Introduction}
\label{sec:Introduction}
Offline RL \cite{levine2020offline} algorithms are able to learn good policies from historical data without any interaction with the environment, which can be costly due to the exploratory nature of online RL.
Ideally, an offline RL algorithm should be able to learn from large datasets to outperform any policy whose state-action distribution is well covered by the data. This is called learning consistency. However, in offline RL, there is no direct feedback from the environment, and instead, the feedback is from the dataset in the form of a ``consistency measure", e.g., Bellman residual \cite{antos2007fitted}. The difficulty with the Bellman residual is that it has high variance due to partial coverage of the data, i.e., when the data-collecting policy does not cover the entire state-action space, which often leads to overestimation of the value functions \cite{jin2021pessimism}. 

Existing works address the overestimation of the value functions either by constraining the policy updates \cite{fujimoto2021minimalist} or by encoding pessimism in the Q-learning updates \cite{cheng2022adversarially}. However, the policy learned using these approaches is not robust to the suboptimality of the learned policy parameters, i.e., any deviation in the policy parameters from the optimal parameters contributes to the suboptimality of the final learned policy and its performance. Moreover, a single set of policy parameters is learned for each time step without any component of forecasts in the policy. In this paper, we investigate \emph{whether the dependence of policy on the suboptimality of the learned parameters can be improved and whether there is any benefit to incorporating a forecasting mechanism in the policy.}

It is known that optimization-based approaches (especially model predictive control (MPC)) are still the de facto standard in many industries. These approaches typically use a lookahead predictive model to compute a generally reasonable policy based on an approximate model. Furthermore, these approaches are often more interpretable and can respect constraints specified on the trajectory. However, optimization-based approaches are prone to model misspecification errors, which can lead to suboptimal decisions.

To leverage the synergistic strength of both optimization-based and pessimistic offline RL methods, we propose an \emph{implicit actor-critic (iAC)} framework that employs optimization solution functions as a policy for the actor and a monotone function over the optimal value of optimization as a critic. By encoding “optimality” in the actor policy, we show that the learned policies are robust to the suboptimality of the learned actor parameters via the exponentially decaying sensitivity (EDS) property \cite{shin2022exponential}. EDS property has been widely established for the finite time optimal control problems (FTOCP) and states that the sensitivity of a solution of an FTOCP at some time step $t'$ against a parameter perturbation at time step $t$ decays exponentially with $|t'-t|$ under certain constraint qualification conditions \cite{shin2022exponential}. 

This EDS property is the main component in iAC. Even if value functions are overestimated, leading to suboptimal actor policy parameters, the decaying sensitivity ensures the resulting policy remains robust against policy parameter mismatch from the optimal. This is in contrast to the general function approximation schemes, where each policy parameter mismatch contributes to the policy mismatch and the final performance. 
Note that this attractive EDS property only holds around a neighborhood of optimal parameters. Therefore, we expand the region of robustness to estimation errors by combining iAC with relative pessimism (RP) \cite{cheng2022adversarially}. 




\emph{Contributions.} First, we propose an offline RL framework: implicit actor-critic (iAC) and show that the policies learned from iAC are robust to the suboptimality of the learned actor parameters due to the EDS property (Lemma \ref{lem:CompGenFuncApprox}). Second, we obtain the convergence rate for iAC (Theorem \ref{thm:iACguarantee}) as $\mathcal{O}\left(\frac{(1-\lambda^{T+1})}{(1-\gamma)(1-\lambda)}+ \frac{1}{(1-\gamma)n^{1/3}} \right)$, where $\lambda\in(0,1)$ is the EDS sensitivity parameter, $\gamma$ is the discount factor, $n$ is the dataset size, and $T$ is the planning horizon. This demonstrates the benefits of the EDS parameter $\lambda$ and an extended prediction horizon $T$ over the $\mathcal{O}\left(\frac{\sqrt{T+1}}{(1-\gamma)}+ \frac{1}{(1-\gamma)n^{1/3}} \right)$ rate obtained for the general function approximation.
Third, we validate the effectiveness of the iAC on two real-world applications showing its benefit over other SOTA offline RL baselines.
\subsection{Related Work}
Many existing works address the partial coverage and the overestimation of value functions by incorporating conservative policy updates. The idea is to enforce the learned policy to be close to the behavior policy by adding constraints in the objective via a policy regularizer \cite{fujimoto2021minimalist} or by using importance sampling with bounded ratio \cite{lee2021optidice}. In contrast, in iAC, there is no need to constrain the policy to be close to the behavior policy as the actor's robustness to the policy parameter mismatch is ensured due to the EDS property.

Other lines of work encode pessimism in the value function updates to overcome overestimation of the value functions \cite{xie2021bellman,cheng2022adversarially}. Relative pessimism (RP) was proposed in \cite{cheng2022adversarially}, which ensures a robust policy improvement over behavior policies. As the EDS property only holds in a local region around the optimal \cite{shin2022exponential}, iAC uses RP to expand this region of robustness to value function estimation errors. We show that due to the combined optimization-based and RP structure of iAC, it is able to achieve strong performance guarantees. Model-based offline RL approaches \cite{yu2021combo} combine pessimism with uncertainty quantification to improve model estimation. In contrast, iAC is an implicit model-based RL approach, where the policy parameters determine both the model and policy, and the parameters are updated not to estimate the model accurately but to optimize for the ultimate objective, i.e., maximize rewards.

Recently, many other works have emerged that learn the optimization policies by learning the parameters of the optimization model. In \cite{agrawal2020learning}, policies are learned by tuning the parameters within the convex optimization under known dynamics. In our framework, we consider unknown dynamics and the offline setting, which is more challenging.

\emph{Notation.} For any positive integers $t_1$, $t_2$, and $T$, we denote $[T] \coloneqq \{1, \cdots, T\}$; $[t_1:t_2] \coloneqq \{t_1,\cdots,t_2\}$; variable sequence $\{x_{t_1}, \cdots, x_{t_2}\}$ as $x_{t_1:t_2}$; $x_i$ as the $i^{th}$ entry of some vector $x$;  $\left\lfloor c \right\rfloor$ and $\lceil c\rceil$ denotes the floor and ceiling function on some scalar $c$, respectively.

\emph{Paper organization.} In Section \ref{sec:preliminaries}, we introduce the preliminaries for offline RL and EDS property. In Section \ref{sec:iAC framework}, we introduce the iAC framework and robustness property of the actor. In Section \ref{sec:TheoreticalAnalysis}, we obtain performance guarantees for the iAC. In Section \ref{sec:experiments} we validate iAC against other offline RL baselines and conclude the paper in Section \ref{sec:Conclusion}.
\section{Preliminaries}
\label{sec:preliminaries}
\subsection{Markov decision process and offline RL}
\label{subsec:MDP}
We consider an infinite horizon Markov decision-process (MDP) denoted by a tuple $(\mathcal{S},\mathcal{A},P,r,\gamma)$, where $\mathcal{S}\subseteq \mathbb{R}^{n_s}$ is the state space, $\mathcal{A}\subseteq \mathbb{R}^{n_a}$ is the action space, ${P}: \mathcal{S} \times \mathcal{A} \rightarrow M(\mathcal{S})$ is the transition function with $P(\cdot \mid s,a)$ denoting the next-state distribution after taking action $a$ at state $s$, and $M(\mathcal{S})$ is the set of all probability measures over the measurable space $\mathcal{S}$, $r:\mathcal{S} \times \mathcal{A}\rightarrow \mathbb{R}$ is the reward function, and $\gamma \in (0,1)$ is the discount factor. We use $\pi:\mathcal{S} \rightarrow \mathcal{A}$ as a deterministic policy mapping from the state space to the action space. The expected discounted return for policy $\pi$ is the expected cumulative rewards when $\pi$ is executed in the environment, i.e., $J(\pi) \coloneqq \mathbb{E}[\sum_{t=0}^\infty\gamma^t r(s_t,\pi(s_t))]$, where $s_{t+1} \sim P(\cdot \mid s_t,\pi(s_t))$. The Q-function for $\pi$ is denoted by $Q^{\pi}: \mathcal{S} \times \mathcal{A} \rightarrow \mathbb{R}$, where $Q^{\pi}(s,a) = \mathbb{E}[\sum_{t=0}^\infty\gamma^tr(s_t,a_t) \mid s_0 = s,a_0 = a]$, where the expectation is taken over the transition dynamics.

We denote the \emph{discounted visitation distribution} associated with policy $\pi$ as $d^\pi(s;\rho) \coloneqq (1-\gamma)\sum_{t=0}^\infty \gamma^t P(s_t = s|s_0 \sim \rho,\pi(s))$, where $\rho$ is the initial state distribution. We denote the Bellman evaluation operator by $\mathcal{T}^{\pi}: B(\mathcal{S} \times \mathcal{A}) \rightarrow B(\mathcal{S} \times \mathcal{A})$ for policy $\pi$ given as:
\begin{equation*}
    (\mathcal{T}^{\pi}Q)(s,a) = r(s,a) + \gamma \int_{\mathcal{S}}Q(s', \pi(s'))P(ds\mid s,a).
\end{equation*}
We denote the optimal policy by $\pi^*$, its corresponding expected return by $J(\pi^*)$ such that $J(\pi^*) = \underset{\pi}{\sup}\hspace{0.1cm}J(\pi)$, and $\nu^*$ as its induced state-action distribution. We further assume that the MDP has smooth dynamics.
\begin{assumption}
    \label{asmptn:SmoothMDP}
    The transition probability is Lipschitz with respect to actions, i.e., $|P(B|s,a_2) - P(B|s,a_1)| \leq L_P\|a_2-a_1\|_2$, where $B$ is a bounded measurable set of $\mathcal{S}$.
\end{assumption}

\emph{Goal in offline RL.} Let the RL agent have access to a dataset $\mathcal{D}$ consisting of $n$-tuples of $(s,a,s',r)$, collected using some behavior policy, where $s'$ denotes the next state. We assume the behavior policy induces a state-action distribution $\mu$. The goal in offline RL is to find a policy $\pi$ that maximizes $J(\pi)$, given access to the dataset $\mathcal{D}$. We define the empirical estimate w.r.t. dataset $\mathcal{D}$ as $\mathbb{E}_{\mathcal{D}}[y] = \frac{1}{n}\sum_{(s_i,a_i,s_i',r_i)\in \mathcal{D}}y((s_i,a_i,s_i',r_i))$ for some function $y$, and the $\nu$-weighted $\mathcal{L}_2$ norm for some function $f$ as $\|f\|_{2,\nu}^2 \coloneqq \int_{\mathcal{S}\times\mathcal{A}}f(s,a)^2 \nu(ds,da)$ for some distribution $\nu$.

\subsection{Optimization solution functions and EDS property}
\label{subsec:LookaheadOptimization}
A widely adopted method for control problems in industries is to formulate the decision-making problem as an approximate receding horizon MPC problem \cite{borrelli2017predictive}, where the future states and actions are captured through constraints as a part of the \emph{lookahead model}. We denote the parameters of the objective function and the constraints for such an MPC model collectively by $\theta \coloneqq \{\zeta_{0},\zeta_{1},\cdots,\zeta_{T}\}$, for finite time horizon $T$, where each $\{\zeta_{i}\}_{i=0}^T \in \mathcal{Z} \subseteq \mathbb{R}^d$ and $\theta \in \Theta \subseteq \mathbb{R}^{d(T+1)}$. To account for infinite horizon decision-making, these MPC problems are usually solved repeatedly for multiple time horizons $T$. We denote the last time step in each instance of MPC by $\tilde{T} \coloneqq T \lceil \frac{t}{T} \rceil$ for any $t \in (0,\infty]$.
We refer the solution of MPC at time $t$ by $\pi_\theta(s_t)$, given as:
\begin{equation}
 \label{eq:actorOptimization}
    \begin{aligned}
      & \underset{\underset{\underset{s_{t+1:\tilde{T}}}{a_{t+1:{\tilde{T}}-1}}}{a_t \in \mathcal{A}}}{\argmax}  \left[\bar{r}(s_t,a_t;\zeta_{\tilde{t}})+ \sum_{t'=t+1}^{\tilde{T}-1} \bar{r}(s_{t'},a_{t'};\zeta_{\tilde{t'}}) +\Bar{R}(s_{\tilde{T}};\zeta_{T})\right] \\ & \text{s. t.} \hspace{0.5cm} g(s_{t'},a_{t'};\zeta_{\tilde{t'}}) \leq 0 \hspace{0.2cm} ; \qquad \quad t' \in [t:\tilde{T}],\\ & \hspace{1cm}
    h(s_{t'},s_{t'+1},a_{t'}; \zeta_{\tilde{t'}}) = 0  \hspace{0.2cm} ; \quad t' \in [t:\tilde{T}],
     \end{aligned}
\end{equation}
where $\tilde{t} = t - T \lfloor \frac{t}{T} \rfloor$ for $\tilde{t} \in [0,T]$, $s_{t+1:\tilde{T}}$ and $a_{t+1:\tilde{T}-1}$ are the decision variables corresponding to the planned states and actions, $\bar{r}:\mathcal{S}\times\mathcal{A}\rightarrow\mathbb{R}$ is the surrogate reward function (not same as true reward $r$), and $\Bar{R}:\mathcal{S}\rightarrow \mathbb{R}$ is the surrogate terminal reward function for the terminal state $s_{\tilde{T}}$.
We assume $\mathcal{Z}$ and $\Theta$ to be compact and convex. 
We refer $\pi_{\theta}(s_t)$ as the solution function \cite{dontchev2009implicit} as it provides the optimal solution to \eqref{eq:actorOptimization}. As this function is generally set-valued \cite{dontchev2009implicit}, we make the following assumption.
\begin{assumption}
\label{asmptn:regularityAssumption}
For each $\theta \in \Theta$ and $s_t \in \mathcal{S}$, the objective function $\bar{r}_t$ in \eqref{eq:actorOptimization} is continuous, strictly convex, $g$ is continuous and convex, and $h$ is affine. Moreover, the feasible set is closed, absolutely bounded, and has a nonempty interior.
\end{assumption}

The above assumption is mild and can be satisfied by imposing proper conditions on the design of the surrogate MPC model by a domain expert. Note that we make no convexity assumption about the true rewards $r$ of the environment. An immediate consequence of the above assumption is that the solution to \eqref{eq:actorOptimization} is unique and continuous with respect to parameters $\theta$ (due to Berge's maximum theorem \cite{berge1997topological}).

\emph{EDS property.} EDS property states that the sensitivity of the solution to an MPC (e.g., \eqref{eq:actorOptimization}) at some time step $t'$ against a parameter perturbation at time step $t$ decays exponentially with $|t-t'|$, where
the sensitivity bounds generally take the following form for some $t, t' \in [0,T]$ and  sensitivity parameters $H\geq1$, $\lambda \in (0,1)$ \cite{shin2022exponential,lin2022bounded}:
\begin{equation}
    \label{eq:EDSsampleEqn}
    \|\pi_{\theta^*}(s_{t'}) - \pi_{\theta'}(s_{t'})\|_2 \leq H\lambda^{|t-t'|}\|\zeta_{t}^*-\zeta_{t}'\|_2,
\end{equation} 
where $\theta^*$ and $\theta'$ denote the list of parameters, which only differ at time step $t$, respectively. The above result shows that the effect of parameter mismatch $\|\zeta_t^*-\zeta_t'\|_2$ at time $t$ on the solution mismatch $\|\pi_{\theta^*}(s_{t'}) - \pi_{\theta'}(s_{t'})\|_2$ at time $t'$ exponentially decays with time difference $|t-t'|$. 

\section{Implicit Actor-Critic (iAC) framework} \label{sec:iAC framework}

\subsection{Actor optimization model}
\label{subsec:ActorModel}
We propose an actor-critic framework \cite{cheng2022adversarially} for offline RL where we consider that the actor's policy comes from the solution of \eqref{eq:actorOptimization}, which we refer to as actor optimization. We name our method ``implicit" because the policy is implicitly
determined by solving an optimization problem. We denote the actor policy class by $\Pi=\{\pi_\theta(\cdot):\theta\in \Theta \subseteq \mathbb{R}^{d(T+1)}\}$. The actor optimization solves a surrogate problem, which is deterministic, and there will be a mismatch between the surrogate model \eqref{eq:actorOptimization} and the true MDP. Therefore, the goal of the offline RL problem is to learn the parameters $\theta$ of the actor optimization model \eqref{eq:actorOptimization} given access to only offline data $\mathcal{D}$, to shape the solution function optimally in order to maximize $J(\pi_\theta)$. We denote such actor parameters by $\theta^*$ and introduce the following definition and assumption.
\begin{definition}[Prediction error]\label{def:predictionError} 
    The $\tau$-step away prediction error for the actor optimization \eqref{eq:actorOptimization} at time step $t$ is defined as $e_{t,\tau} \coloneqq \|\zeta_{t|t+\tau}-\zeta_{t|t+\tau}^*\|_2$ for some integer $\tau\geq0$, where $\zeta_{t|t+\tau}$ denotes the $\tau$-step away parameter from $\tilde{t}$.
\end{definition}
\begin{assumption}
    \label{asmptn:LICQasmptn} The actor optimization satisfies the following regularity conditions at optimal parameter $\theta^*$:\emph{1)} linear independence constraint qualification (LICQ); and \emph{2)} strong second-order sufficiency condition (SSOSC).
\end{assumption}

The assumption above is required to have the EDS property hold in some neighborhood around $\theta^* \in \Theta$, as in optimization problems with inequality constraints, one cannot guarantee EDS to hold globally around $\theta^*$ \cite{shin2022exponential}. It can be satisfied by imposing proper conditions on the design of actor optimization. 
Next, we describe the critic model to learn $\theta^*$.

\subsection{Critic for Q-function approximation}
\label{subsec:CriticModel}
We use the optimal value from the actor optimization at $t=0$ to approximate the Q-value at state $s_0=s$ for action $a$. We propose a reward warping function (RWF) that takes as input the optimal value from the actor given by: 
\begin{equation}
    \begin{aligned}
     & \phi(s,a;\theta') =\underset{s_{0:T},a_{0:T-1}}{\max}\Bigg[ \sum_{t'=0}^{T-1} \bar{r}(s_{t'},a_{t'};\zeta_{t'}') + \Bar{R}(s_T;\zeta_{T}')\Bigg] \\ & \text{s. t.} \hspace{0.5cm} g(s_{t'},a_{t'};\zeta_{t'}') \leq 0 \hspace{0.2cm} ; \qquad \quad t' \in [0:T],\\ & \hspace{1cm}
    h(s_{t'},s_{t'+1},a_{t'}; \zeta_{t'}') = 0  \hspace{0.2cm} ; \quad t' \in [0:T],\\ & \hspace{1cm}s_0 = s, a_0 = a.
     \end{aligned}
     \label{eq:CriticOptimization}
\end{equation}
The critic's output is given as $f(s,a;\omega,\theta') \coloneqq \psi\big(\phi(s,a;\theta');\omega \big)$,
where $\psi(\cdot;\omega):\mathbb{R}\rightarrow\mathbb{R}$ is the RWF parameterized by some parameters $\omega \in \Omega$. Note that $\theta'$ above may not be the same as the actor's parameter $\theta$. We denote the the critic function class by $\mathcal{F}\coloneqq \{f(\cdot,\cdot;\omega,\theta'): \theta' \in \mathbb{R}^{d(T+1)}, \omega \in \Omega\}$. We denote the maximum value of critic with $\psi_{max} \coloneqq \max_{(s,a)}\psi(s,a)$. For notational simplicity, we denote $f(s,a;\omega,\theta')$ as $f_\omega^{\theta'}(s,a)$. 

The design of RWF should facilitate the prediction of the true rewards-to-go while preserving some properties of the underlying optimization. 
We consider the RWF a \emph{monotonically increasing function} as they can have better statistical properties compared to non-monotone counterparts due to the structural prior \cite{monotoneNN}. We introduce the following assumptions for the critic function class $\mathcal{F}$.
\begin{assumption}
    \label{asm:PsiLipschitz}
    We assume that the monotone function $\psi:\mathbb{R} \rightarrow \mathbb{R}$ is Lipschitz with respect to the actions, i.e., 
        $|f_\omega^\theta(s,a_2) - f_\omega^\theta(s,a_1)| \leq L_{\psi}\|a_2 - a_1\|_2$ for all $s \in \mathcal{S}$.
\end{assumption}
\begin{assumption}\label{asmptn:realizability} For any $\pi \in \Pi$, we assume approximate realizability for $f_\omega^\theta\in \mathcal{F}$, i.e., $\min_{f_\omega^\theta\in \mathcal{F}}\max_{\nu} \|f_\omega^\theta-\mathcal{T}^\pi f_\omega^\theta\|_{2,\nu}^2 \leq \epsilon_{\mathcal{F}}$, where $\nu \in \{d^{\pi}: \forall \pi \in \Pi\}$. 
\end{assumption}
 Intuitively, $\epsilon_{\mathcal{F}}$ is the maximum approximation error that the critic function class $f_\omega^\theta \in \mathcal{F}$ will have to approximate some $Q^{\pi_\theta}$.
Assumptions \ref{asm:PsiLipschitz} and \ref{asmptn:realizability} are mild as the function $\psi$ can be designed to ensure the above property holds.

\subsection{Offline iAC with relative pessimism}
\label{subsec:iACwithRelativePessimism}
As the EDS conditions only hold around a neighborhood of $\theta^*$, we could only have its benefits inside a local region. Therefore, we expand the region of robustness to overestimation of the value functions by combining iAC with RP \cite{cheng2022adversarially}, which enables iAC to address partial coverage while ensuring robust policy improvement over any behavior policy. 
RP formulates offline RL as a Stackelberg game with the actor as the leader and the critic as the follower \cite{cheng2022adversarially}: 
\begin{equation}
    \begin{aligned}\label{eq:ATAC1}
        &\hat{\theta} \in \argmax_{\theta \in \Theta} \mathcal{L}_{\mu}\left(\theta, \hat{\omega}\right) \\ \text{s.t.} \quad & \hat{\omega} \in \argmin_{\omega\in\Omega} \mathcal{L}_{\mu}\left(\theta,\omega\right) + \beta\mathcal{E}_\mu(\theta,\omega),
    \end{aligned}
\end{equation}
where $\beta \geq 0$ is a hyperparameter, and
\begin{equation}\label{eq:ATAC2}
    \mathcal{L}_\mu(\theta,\omega) \coloneqq \mathbb{E}_{(s,a)\sim\mu}\left[f^{\theta}_\omega(s,\pi_\theta(s)) - f^{\theta}_\omega(s,a)\right],
\end{equation}
\begin{equation}\label{eq:ATAC3}
    \mathcal{E}_{\mu}(\theta,\omega) \coloneqq \mathbb{E}_{(s,a)\sim\mu}\left[\left((f_\omega^\theta - \mathcal{T}^{\pi_\theta}f_\omega^\theta)(s,a)\right)^2\right].
\end{equation}
Here, $\pi_{\hat{\theta}}$ maximizes the relative pessimistic policy evaluation of $f^{\theta}_\omega$ with respect to the behavior policy distribution $\mu$. The term $\mathcal{E}_\mu(\theta,\omega)$ ensures Bellman consistency on data for $f^{\theta}_\omega$, $\mathcal{L}_\mu(\theta,\omega)$ encodes pessimism in the updates, $\beta$ balances the relative contributions of $\mathcal{E}_\mu(\theta,\omega)$ and $\mathcal{L}_\mu(\theta,\omega)$. 

\subsection{iAC algorithm and robustness of the actor} 
We introduce the overall iAC algorithm with RP. First, we initialize actor policy with parameters $\theta^{(0)} = \{\zeta_0^{(0)},  \cdots, \zeta_T^{(0)}\}$, where the superscript denotes the training iteration. At iteration, $k\in[0:K]$, the pessimistic estimate of the critic is obtained as 
\begin{equation}
\label{eq:criticUpdate}
    \omega^{(k+1)} \leftarrow \argmin_{\omega \in \Omega}\mathcal{L}_{\mathcal{D}}(\theta^{(k)},\omega) + \beta \mathcal{E}_\mathcal{D}^{\omega^{(k)}}(\theta^{(k)},\omega),
\end{equation}
and then, the actor parameters are updated as
\begin{equation}
    \label{eq:actorUpdate}
    \theta^{(k+1)} \leftarrow \argmax_{\theta \in \Theta} \mathcal{L}_{\mathcal{D}}(\theta,\omega^{(k+1)}),
\end{equation}
where $\mathcal{L}_{\mathcal{D}}(\theta,\omega) \coloneqq \mathbb{E}_{\mathcal{D}}[f_\omega^\theta(s,\pi_\theta(s)) - f_\omega^\theta(s,a)]$ and $\mathcal{E}_{\mathcal{D}}^{\omega^{(k)}}(\theta,\omega) \coloneqq [f_\omega^\theta(s,a) - r - \gamma f_{\omega^{(k)}}^\theta(s',\pi_\theta(s))]^2$ are the empirical estimates of $\mathcal{E}_\mu$ and $\mathcal{L}_\mu$, respectively. The next proposition shows how the learned actor is robust to the learned parameter mismatch after $K$ training iterations.
\begin{proposition}[\cite{lin2022bounded}]
    \label{prop:RobustActor}
    Let $e_{t,\tau}^{(K)}$ denote the prediction error at iteration $K$ at some time step $t$. If Assumptions \ref{asmptn:SmoothMDP}-\ref{asmptn:LICQasmptn} hold, and if there exist constants $H\geq 1$, $C_1>0$, and $\lambda \in (0,1)$ such that: $\sum_{\tau=0}^{T-t}\lambda^\tau e_{t,\tau}^{(K)}\leq \left(\frac{(1-\lambda)^2}{H^3 L_P} - 2C_1\lambda^{T-t}\right)$ for all $t\in [0:T]$, then the following holds $\forall s_0 \in \mathcal{S}$:
    \begin{equation*}
\begin{aligned}
\|\pi_{\theta^*}(s_0) - \pi_{\theta^{(K)}}(s_0)\|_2\leq\mathcal{O}\left(\sum_{t=0}^T\lambda^tH\|\zeta_t^* - \zeta_t^{(K)}\|_2  \right).
\end{aligned}
\end{equation*}
\end{proposition}

The above proposition follows from \cite[Thm. 5.1]{lin2022bounded}. It states that for optimization-based policies, under bounded prediction errors $e_{t,\tau}^{K}$, the EDS property holds, i.e., the behavior mismatch of the learned policy ($\|\pi_{\theta^*}(s_0) - \pi_{\theta^{(K)}}(s_0)\|_2$) has an exponentially decaying dependence on the parameter mismatch later in the time horizon. This is in contrast to the general function approximation settings, where each parameter mismatch contributes to the suboptimality of the learned policy without any exponentially decaying weighting. To show this, we consider $\theta^{gen}\in \Theta^{gen} \subseteq \mathbb{R}^{d(T+1)}$ to denote the policy parameter for some general function approximation for the actor policy learned after $K$ iterations. We assume that the policy $\pi_{\theta^{gen}}(s)$ is Lipschitz continuous in $\theta^{gen}$. 
\begin{lemma}
    \label{lem:CompGenFuncApprox}
    Let $$Z = \left(\max_{i\in[d(T+1)]}|(\theta^* -\theta^{gen})_i|, \max_{i\in[d(T+1)]}|(\theta^* - \theta^{iAC})_i| \right),$$ denote the maximum scalar entry in both parameter mismatch vectors. If Proposition \ref{prop:RobustActor} holds, then for all $s \in \mathcal{S}$
    \begin{equation}
    \label{eq:GenFuncApprox}
        \|\pi_{\theta^*}(s) - \pi_{\theta^{gen}}(s)\|_2 \leq \mathcal{O}\left(Z\sqrt{d(T+1)}\right),
        \end{equation}
        \begin{equation}
        \label{eq:iACrobust}
        \|\pi_{\theta^*}(s) - \pi_{\theta^{iAC}}(s)\|_2 \leq \mathcal{O}\left(Z\sqrt{d}\frac{(1-\lambda^{T+1})}{1-\lambda}\right),
    \end{equation}
    where $\lambda \in (0,1)$ is the constant from Proposition \ref{prop:RobustActor}.
\end{lemma}

 The above lemma shows that under a longer horizon $T$ and large exponential decay parameter $\lambda$, the policy mismatch for iAC has a tighter upper bound with respect to the parameter mismatch compared to the general function approximation settings, highlighting the robustness of the actor in iAC.
\section{Theoretical analysis for iAC}
\label{sec:TheoreticalAnalysis}
In this section, we provide theoretical guarantees on the performance of iAC. 
We introduce the following definition.

\begin{definition}[\cite{xie2021bellman}]
    Distribution shift coefficient (DSC) $\mathcal{C}(\nu;\mu,\mathcal{F},\pi) \coloneqq \underset{f \in \mathcal{F}}{\max}\frac{\|f - \mathcal{T}^{\pi} f\|_{2,\nu}^2}{\|f - \mathcal{T}^{\pi} f\|_{2,\mu}^2}$ measures the distribution shift from an arbitrary distribution $\nu$ to the distribution $\mu$ with respect to the function class $\mathcal{F}$ and policy $\pi$.
\end{definition}

DSC measures how well the behavior policy distribution $\mu$ covers the distribution of interest $\nu$ with respect to the function class $\mathcal{F}$. 
We introduce the following assumption.
\begin{assumption}
    \label{asmptn:boundedC}
    For some $C^*> 0$, the following holds 
     $\max_{k\in [0:K]}\mathcal{C}(\nu^*;\mu,\mathcal{F},\pi_{\theta^{(k)}})\leq C^*$. 
\end{assumption}

The assumption above implies that the distribution $\nu^*$ is well-covered by $\mu$ under $\mathcal{F}$ and any $\pi_{\theta^{(k)}}$. This is a much weaker assumption than the boundedness of the all-policy concentrability ratio, which assumes the boundedness for all policies $\pi \in \Pi$ (see, e.g., \cite{antos2007fitted}). We use $S_{\mathcal{F},\Pi}\coloneqq \frac{\log \left(\mathcal{N}_1(\mathcal{F}, \epsilon)\mathcal{N}_1(\Pi, \epsilon)\right)}{\delta}$ to denote the joint statistical complexity of the actor class $\Pi$ and critic class $\mathcal{F}$, where $\delta\in(0,1)$ is some probability, and $\mathcal{N}_1(\mathcal{F},\epsilon)$, $\mathcal{N}_1(\Pi,\epsilon)$  denote the $\mathcal{L}_1$-covering numbers at precision $\epsilon$ (see \cite{kakade2008complexity}). 

\subsection{Performance guarantee for iAC}
\label{subsec:PerfGuarantee}
For the performance guarantee, we obtain an upper bound on $J(\pi_{\theta^*}) - J(\pi_{\theta^{(K)}})$, which can be decomposed as \cite{cheng2022adversarially}:
\begin{equation*}
\small
\begin{aligned}
    J(\pi_{\theta^*}) - J(\pi_{\theta^{(K)}}) &= \frac{1}{1-\gamma}\Bigg[\underbrace{\mathbb{E}_\mu\left[\left(f_K - \mathcal{T}^{\pi_{\theta^{(K)}}}f_K\right)(s,a)\right]}_{\text{(A)}}  \\ &+\underbrace{\mathbb{E}_{d^{\pi_{\theta^*}}}\left[\left(\mathcal{T}^{\pi_{\theta^{(K)}}}f_K - f_K\right)(s,a)\right]}_\text{(B)} \\&+\underbrace{\mathbb{E}_{d^{\pi_{\theta^*}}}\left[f_K(s,\pi_{\theta^*}(s)) - f_K(s,\pi_{\theta^{(K)}}(s))\right]}_\text{(C)} \\ &+ \underbrace{\mathcal{L}_\mu(\theta^{(K)},\omega^{(K)}) -\mathcal{L}_\mu(\theta^{(K)},Q^{\pi_{\theta^{(K)}}})}_\text{(D)}\Bigg],
    \end{aligned}
\end{equation*}
where we denote $f_{\omega^{(K)}}^{\theta^{(K)}}$ as $f_K$ for notational simplicity. Terms (A) and (B) are the average Bellman errors for $f_K$ and $\pi_{\theta^{(K)}}$, w.r.t. $\mu$ and $d^{\pi_{\theta^*}}$; term (D) is the difference between true and estimated relative pessimistic Q-values. Terms (A), (B), and (D) can be upper bounded same way as in \cite[Thm. 14]{cheng2022adversarially}. 

Term (C) denotes the optimization regret incurred during the actor update \eqref{eq:actorUpdate}.
In \cite{cheng2022adversarially}, authors rely on a no-regret policy optimization oracle to achieve $o(K)$ regret on (C) over $K$ iterations, which is hard to achieve for general function approximation settings. For iAC, $o(K)$ regret on the term (C) can be assured by establishing that the critic class $\mathcal{F}$ is piecewise smooth in $\theta$. Many works have achieved $o(K)$ regret for online learning on piecewise smooth functions \cite{balcan2018dispersion}. The piecewise-smooth property for critic follows from two points: \emph{1)} the output of a monotone function applied to the optimal value is equal to the optimal value of the optimization where the monotone function is applied to the objective function (Lemma 2 in appendix) which makes the function class $\mathcal{F}$ a set of optimal-value mappings; \emph{2)} multiparametric (mp) MPC problems can be approximated by multi-stage (mp-quadratic programs (mp-QP)) or mp-convex problems \cite{kouramas2011explicit}, and it is known that the optimal values of mp-QP/convex problems are piecewise smooth w.r.t. the optimization parameters \cite{kouramas2011explicit}. 

We further show that the upper bound on (C), in terms of learned parameter mismatch, improves due to the EDS property. We present the final performance guarantee for the iAC for the case when the actor optimization \eqref{eq:actorOptimization} is represented with a linear surrogate reward. Then \eqref{eq:actorOptimization} can be represented as a multi-stage mp-LP with $n_{\mathcal{I}}$ and $n_{\mathcal{E}}$ inequality and equality constraints, respectively, at each time step \cite{kouramas2011explicit}. We assume $\epsilon_{\mathcal{F}} = 0$ for simplicity.
\begin{theorem}
\label{thm:iACguarantee}
    Under Assumptions \ref{asmptn:SmoothMDP}- \ref{asmptn:boundedC}, if $\beta = \mathcal{O}\left(\frac{\psi_{max}^{1/3}n^{2/3}\delta^{2/3}}{C_{\mathcal{F},\Pi}^{2/3}} \right)$, then with probability of at least $1-\delta$ 
    \begin{equation*}
    \begin{aligned}
        J(\pi_{\theta^*}) - J(\pi_{\theta^{(K)}}) &\leq \mathcal{O}\left(\frac{Z \sqrt{d}(1-\lambda^{T+1})}{(1-\lambda)(1-\gamma)}\right)\\ &+\mathcal{O}\left( \frac{\psi_{max}\sqrt{C^*}\left(C_{\mathcal{F},\Pi}^{1/3} \right)}{(1-\gamma)n^{1/3} \delta^{1/3}} \right),
        \end{aligned}
    \end{equation*}
where $C_{\mathcal{F},\Pi} = \sum_{0\leq i\leq (T+1)(n_a - n_{\mathcal{E}})}\begin{pmatrix}(T+1)n_{\mathcal{I}}\\i\end{pmatrix}$ denotes the constant term in joint statistical complexity $S_{\mathcal{F},\Pi}$.
\end{theorem}

The first term in the RHS provides an improved upper bound on the term (C), highlighting the benefit of the EDS parameter $\lambda$ and an extended prediction horizon $T$. In standard function approximation, this term reduces to $\mathcal{O}\left(\frac{Z\sqrt{d(T+1)}}{(1-\gamma)} \right)$. The $n^{1/3}$ dependence is attributed to the squared Bellman error regularizer in RP. A lower DSC $C^*$ means $d^{\pi_{\theta^*}}$ is well-covered by $\mu$, leading to better performance. The constant $C_{\mathcal{F},\Pi}$ denotes the maximum number of critical regions in reformulated multi-stage mp-LP for actor optimization with linear rewards and constraints \cite{kouramas2011explicit}. We assume $n_a - n_{\mathcal{E}}\geq 0$ for the reformulated mp-LP. As $n_{\mathcal{I}}$ and $T$ increase, $C_{\mathcal{F},\Pi}$ increases. However, increasing $n_{\mathcal{I}}$ also improves the approximation power of $\Pi$ and $\mathcal{F}$ \cite{jin2023solution}, leading to a lower realizability error $\epsilon_{\mathcal{F}} = \mathcal{O}\left( \frac{n_s}{(n_{\mathcal{I}})^{2/n_s}}\right)$ (Lemma 3 in appendix). Therefore, there is a constant tradeoff between the first and second terms in the upper bound through $n_{\mathcal{I}}$ and $T$.




\section{Experiments}
\label{sec:experiments}
\subsection{Building energy management---CityLearn challenge}
We consider an offline version from the CityLearn challenge \cite{vazquez2020citylearn}, which consists of controlling the energy storage of multiple buildings to minimize costs such as ramping costs and peak demands. The state space consists of 30 continuous states (e.g., electric demands, outdoor temperature), and actions include charging/discharging of storage devices for hot water, cold water, and electricity for 9 buildings. 
For experiments, we collect hourly data for 3 months using a rule-based control (RBC), a default controller in the CityLearn environment.
We use the per-step reward of $e_t^3$ as done in \cite{khattar2023winning}, where $e_t$ is the net electricity consumption. 

We formulate the actor optimization with a $T=24$ hour planning horizon; the surrogate reward is designed as $\bar{r}(s_t;\zeta) = - |e_t - e_{t-1}| - \zeta_t e_t$, which consists of the sum of the negated ramping cost and ``virtual" electricity cost, where $\theta = \{\zeta_0,\cdots,\zeta_{24}\}$ and $\zeta_t\in [0,5] \hspace{0.1cm}\forall t \in [0:24]$ are the virtual electricity prices to be learned in order to encourage desirable consumption patterns. We design the RWF for the critic as $\psi(\phi;\omega) = \omega_1\phi+ \omega_2$, where $\phi$ is the optimal value from the actor optimization.

We compare iAC with other famous offline RL baselines: BEAR \cite{kumar2019stabilizing}, AWAC \cite{nair2020awac}, CQL \cite{kumar2020conservative}, TD3+BC  \cite{fujimoto2021minimalist}, and BCQ \cite{fujimoto2019off}. We consider 3 versions of iAC: \emph{1)} iAC* (without RP), \emph{2)} iAC+RP (with RP), \emph{3)} W-iAC (training iAC with RP and by weighting current parameters more than the later ones). 
\begin{figure}[t]
  \centering
\includegraphics[width=0.8\columnwidth]{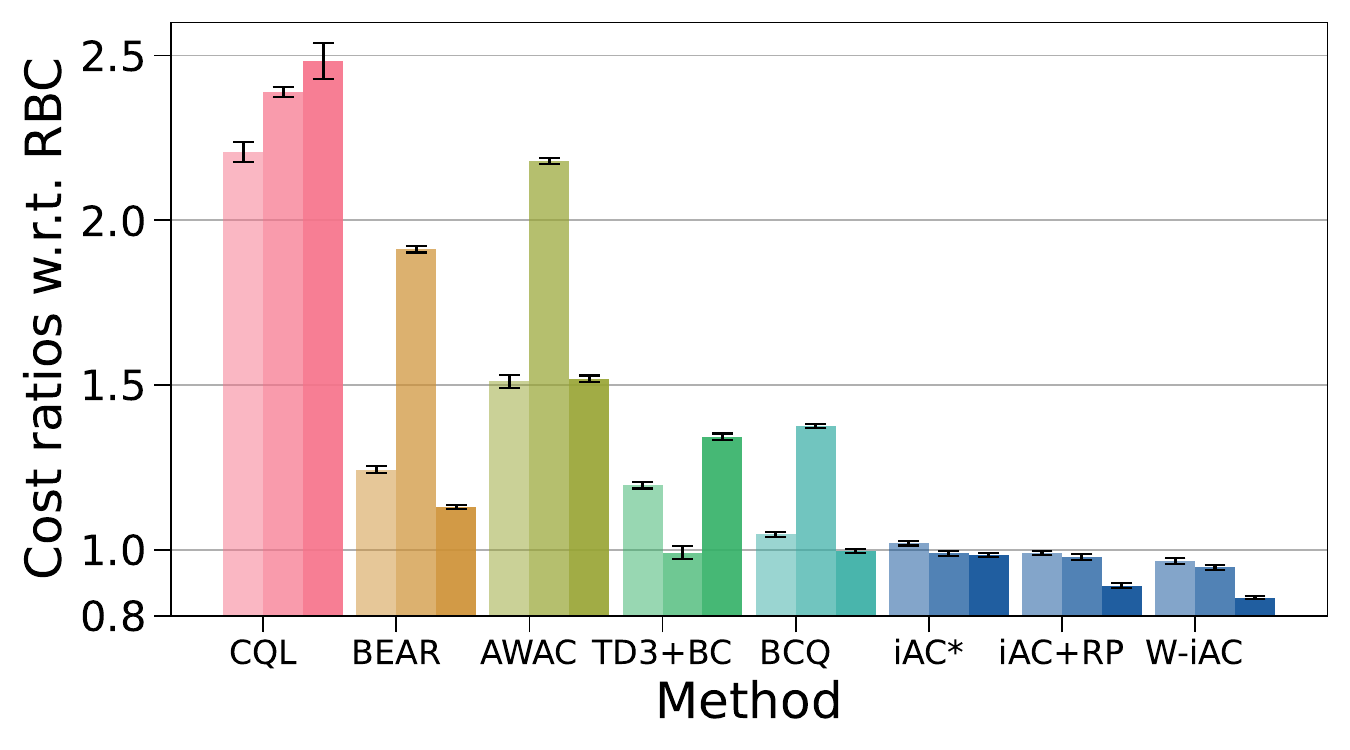}
\vspace{-0.4cm}
\caption{iAC performance comparison, with total cost  (left bar), peak demand (middle bar), and ramping cost (right bar).  Lower is better. Average scores are reported across 10 runs with standard deviation as error bars.}
\label{fig:iAC_bargraph1}
\end{figure}
\begin{figure}[t]
  \centering
\includegraphics[width=0.75\columnwidth]{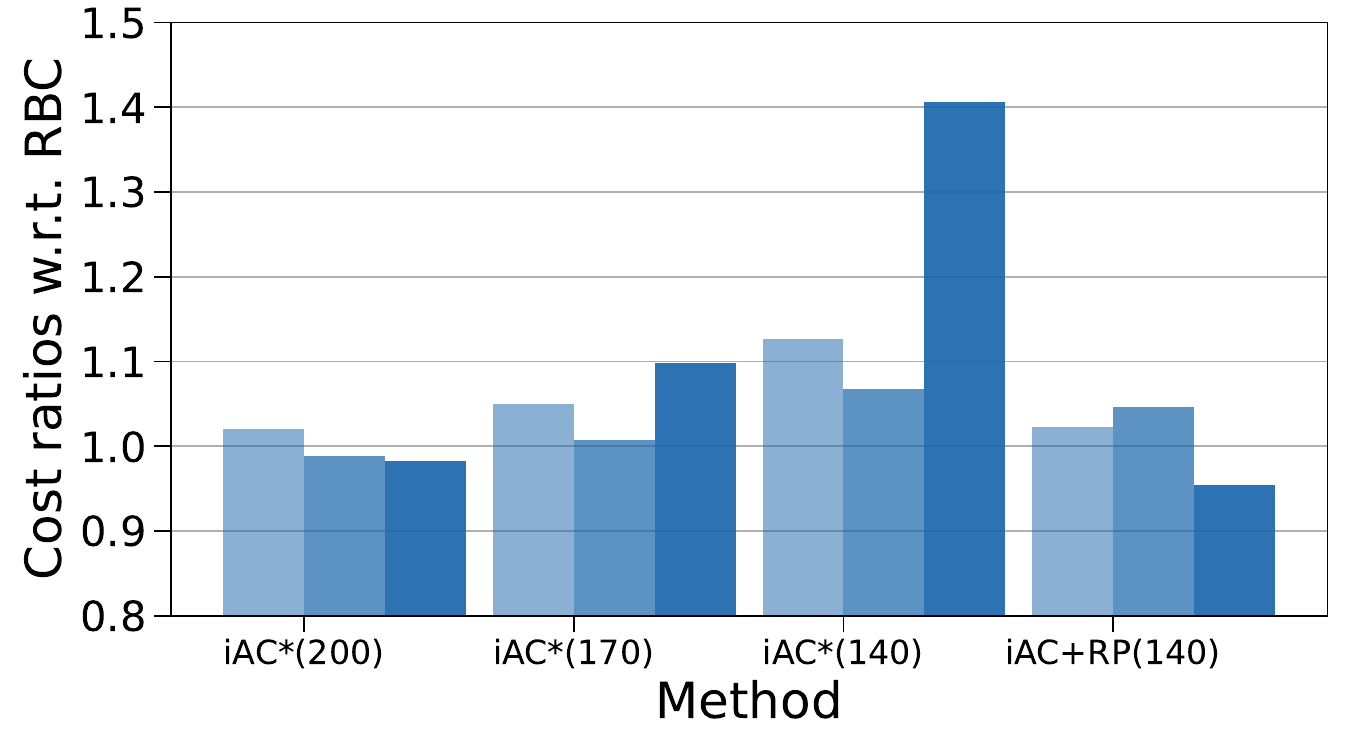}
\caption{Performance of different versions of iAC at different iterations.}
\vspace{-0.4cm}
\label{fig:iAC_robust}
\end{figure}

\emph{Evaluation metrics:} We evaluate the trained offline RL agent on a 1-year stationary dataset from climate zone 1 and report 3 cost ratios w.r.t. RBC: \emph{1)} true cost from the environment, \emph{2)} peak demand, \emph{3)} ramping cost.

Figure \ref{fig:iAC_bargraph1} shows that all 3 iAC versions outperform offline RL baselines on all 3 metrics. Moreover, costs decrease after incorporating RP into iAC* (i.e., iAC+RP) and weighting the current parameter more (i.e., W-iAC). Almost similar performances of iAC* and iAC+RP suggest that the EDS property is being met and not much benefit of RP coming into play. A substantial decrease in cost ratios for W-iAC shows the importance of estimating the current parameters more accurately than parameters later in the time horizon. Figure \ref{fig:iAC_robust} analyzes the robustness of the actor and the benefit of combining EDS with RP. First, almost similar performances of iAC*(200) and iAC*(170) can be seen at training iterations 200 and 170. This can be attributed to the learned parameters inside the EDS region at both iterations, showing robustness to the learned parameter misspecifications. The performance of iAC*(140) got worse due to the learned parameter not entering the local region of EDS, because of relatively fewer training iterations. The inclusion of RP in iAC+RP(140) at 140 iterations reduces costs, showing the benefit of incorporating RP with iAC and having good performance guarantees even outside the EDS region. 

\subsection{Supply chain management}
\label{subsec:SupplyChain}
\begin{figure}[t]
  \centering
\includegraphics[width=0.75\columnwidth]{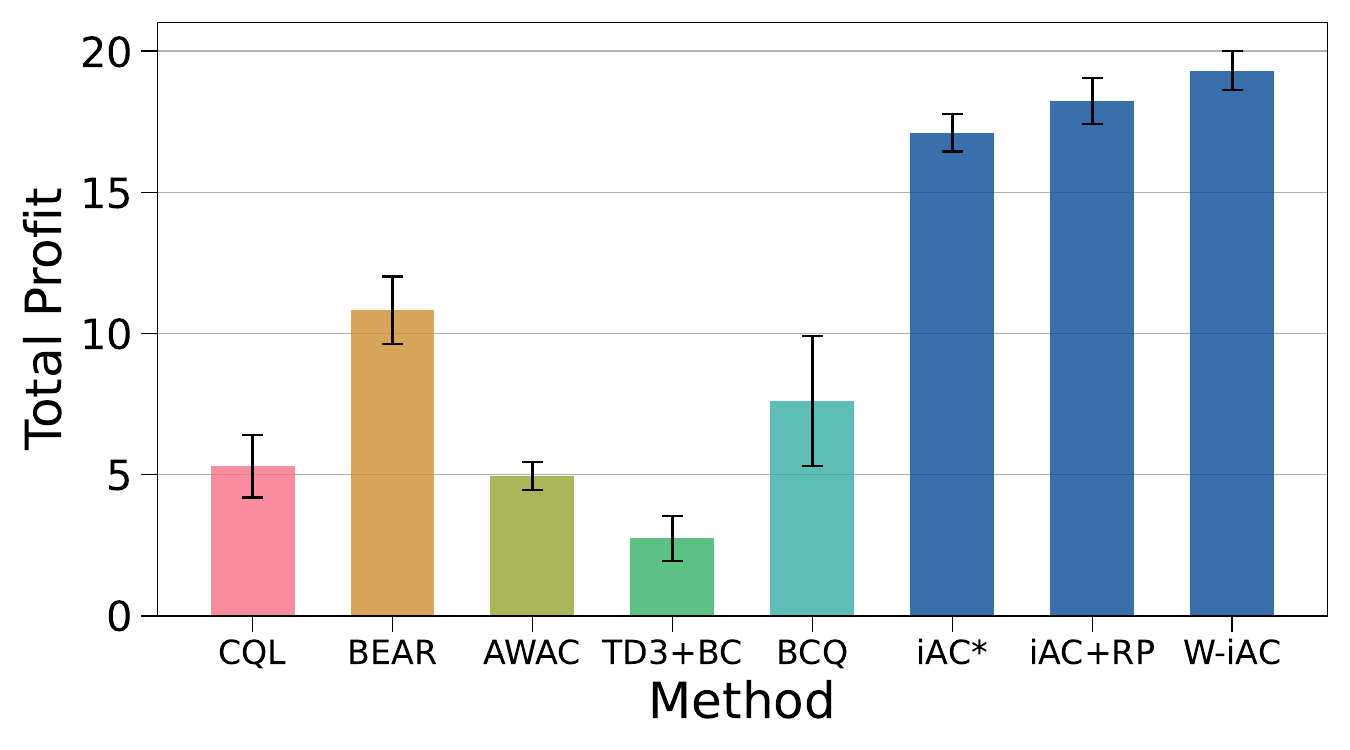}
\vspace{-0.4cm}
\caption{Supply chain total profits and comparison with other offline RL baselines. Average scores and error bars are reported across 10 runs.}
\label{fig:SupplyChainPerformance}
\end{figure}
We consider an offline version of the supply chain problem from \cite{agrawal2020learning}, where the goal is to find a set of policies to maximize profit for a network of supply chains. We consider a setup with prediction horizon $T=20$, $n_h = 4$ nodes, $n_p = 8$ links, and $n_c = 2$ customer links. We collect data using a default optimization model for 200 days. We denote $h_t$ as the amount of goods at each node at time $t$, $p_t$ as the price at which the warehouses can buy from suppliers at time $t$, $Y$ as the fixed price, $d^c_t$ as the consumer demand at time $t$, and $\kappa$ as the shipping cost; actions include \emph{1)} $b_t$ (quantity to buy from suppliers); \emph{2)} $g_t $ (quantity to be sold to the customers); and \emph{3)} $z_t$ (quantity to be shipped). For some random $\alpha, \Lambda$, the holding costs are $\alpha^Th_t + \Lambda^T h_t^2$, where the square is elementwise. 
The future demand and prices are modeled using the log-normal distribution. The true reward is: $ -p_t^Tb_t + Y^Tg_t - \kappa^Tz_t - \alpha^Th_t - \Lambda^Th_t^2$. The actor optimization model is parametrized by $\zeta_t = (G,q) \hspace{0.1cm} \forall t \in [0:20]$, where $G \in \mathbb{R}^{4 \times 4}$, $q \in \mathbb{R}^{4}$, and the surrogate reward is: $\bar{r}(s_t,a_t;G,q) = -p_t^Tb_t - Y^Ts_t + \kappa^Tz_t - \|Gh_t\|^2_2 - q^Th_t$, where $G$ and $q$ capture the effect of randomness in $\alpha$ and $\Lambda$. RWL for the critic is designed as $\psi(\phi;\omega) \coloneqq \omega_1 \phi + \omega_{bias}$, where $\phi$ is the optimal value from the actor.
Figure \ref{fig:SupplyChainPerformance} shows all 3 variants of iAC outperform other baselines in profit maximization. (See Appendix for more experimental details.) 

\section{Conclusion}
\label{sec:Conclusion}
We proposed a novel offline RL framework: implicit actor-critic, which employs optimization solution functions as policies. We showed its benefits theoretically and experimentally by exploiting the EDS property. 

\bibliographystyle{plain}
\bibliography{references}

\begin{thebibliography}{10}

\bibitem{agrawal2020learning}
Akshay Agrawal, Shane Barratt, Stephen Boyd, and Bartolomeo Stellato.
\newblock Learning convex optimization control policies.
\newblock In {\em Learning for Dynamics and Control}, pages 361--373. PMLR, 2020.

\bibitem{antos2007fitted}
Andr{\'a}s Antos, Csaba Szepesv{\'a}ri, and R{\'e}mi Munos.
\newblock Fitted q-iteration in continuous action-space mdps.
\newblock {\em NeurIPS}, 20, 2007.

\bibitem{berge1997topological}
Claude Berge.
\newblock {\em Topological Spaces: including a treatment of multi-valued functions, vector spaces, and convexity}.
\newblock C. Corporation, 1997.

\bibitem{borrelli2017predictive}
Francesco Borrelli, Alberto Bemporad, and Manfred Morari.
\newblock {\em Predictive control for linear and hybrid systems}.
\newblock Cambridge University Press, 2017.

\bibitem{cheng2022adversarially}
Ching-An Cheng, Tengyang Xie, Nan Jiang, and Alekh Agarwal.
\newblock Adversarially trained actor critic for offline reinforcement learning.
\newblock In {\em ICML}, pages 3852--3878. PMLR, 2022.

\bibitem{dontchev2009implicit}
Asen~L Dontchev and R~Tyrrell Rockafellar.
\newblock {\em Implicit functions and solution mappings}, volume 543.
\newblock Springer, 2009.

\bibitem{fujimoto2021minimalist}
Scott Fujimoto and Shixiang~Shane Gu.
\newblock A minimalist approach to offline reinforcement learning.
\newblock {\em Advances in neural information processing systems}, 34:20132--20145, 2021.

\bibitem{fujimoto2019off}
Scott Fujimoto, David Meger, and Doina Precup.
\newblock Off-policy deep reinforcement learning without exploration.
\newblock In {\em International conference on machine learning}, pages 2052--2062. PMLR, 2019.

\bibitem{jin2023solution}
Ming Jin, Vanshaj Khattar, Harshal Kaushik, Bilgehan Sel, and Ruoxi Jia.
\newblock On solution functions of optimization: Universal approximation and covering number bounds.
\newblock In {\em Proceedings of the AAAI Conference on Artificial Intelligence}, volume~37, pages 8123--8131, 2023.

\bibitem{jin2021pessimism}
Ying Jin, Zhuoran Yang, and Zhaoran Wang.
\newblock Is pessimism provably efficient for offline rl?
\newblock In {\em ICML}, pages 5084--5096. PMLR, 2021.

\bibitem{kakade2008complexity}
Sham~M Kakade, Karthik Sridharan, and Ambuj Tewari.
\newblock On the complexity of linear prediction: Risk bounds, margin bounds, and regularization.
\newblock 2008.

\bibitem{khattar2023winning}
Vanshaj Khattar and Ming Jin.
\newblock Winning the citylearn challenge: Adaptive optimization with evolutionary search under trajectory-based guidance.
\newblock In {\em Proceedings of the AAAI Conference on Artificial Intelligence}, volume~37, pages 14286--14294, 2023.

\bibitem{kouramas2011explicit}
Konstantinos~I Kouramas, Nuno~P Fa{\'\i}sca, Christos Panos, and Efstratios~N Pistikopoulos.
\newblock Explicit/multi-parametric model predictive control (mpc) of linear discrete-time systems by dynamic and multi-parametric programming.
\newblock {\em Automatica}, 47(8):1638--1645, 2011.

\bibitem{kumar2019stabilizing}
Aviral Kumar, Justin Fu, Matthew Soh, George Tucker, and Sergey Levine.
\newblock Stabilizing off-policy q-learning via bootstrapping error reduction.
\newblock {\em Advances in neural information processing systems}, 32, 2019.

\bibitem{kumar2020conservative}
Aviral Kumar, Aurick Zhou, George Tucker, and Sergey Levine.
\newblock Conservative q-learning for offline reinforcement learning.
\newblock {\em Advances in Neural Information Processing Systems}, 33:1179--1191, 2020.

\bibitem{lee2021optidice}
Jongmin Lee, Wonseok Jeon, Byungjun Lee, Joelle Pineau, and Kee-Eung Kim.
\newblock Optidice: Offline policy optimization via stationary distribution correction estimation.
\newblock In {\em International Conference on Machine Learning}, pages 6120--6130. PMLR, 2021.

\bibitem{levine2020offline}
Sergey Levine, Aviral Kumar, George Tucker, and Justin Fu.
\newblock Offline reinforcement learning: Tutorial, review, and perspectives on open problems.
\newblock {\em arXiv preprint arXiv:2005.01643}, 2020.

\bibitem{lin2022bounded}
Yiheng Lin, Yang Hu, Guannan Qu, Tongxin Li, and Adam Wierman.
\newblock Bounded-regret mpc via perturbation analysis: Prediction error, constraints, and nonlinearity.
\newblock {\em NeurIPS}, 35:36174--36187, 2022.

\bibitem{balcan2018dispersion}
Maria-Florina, Travis Dick, and Ellen Vitercik.
\newblock Dispersion for data-driven algorithm design, online learning, and private optimization.
\newblock In {\em 2018 IEEE 59th Annual Symposium on Foundations of Computer Science (FOCS)}, pages 603--614. IEEE, 2018.

\bibitem{nair2020awac}
Ashvin Nair, Abhishek Gupta, Murtaza Dalal, and Sergey Levine.
\newblock Awac: Accelerating online reinforcement learning with offline datasets.
\newblock {\em arXiv preprint arXiv:2006.09359}, 2020.

\bibitem{shin2022exponential}
Sungho Shin, Mihai Anitescu, and Victor~M Zavala.
\newblock Exponential decay of sensitivity in graph-structured nonlinear programs.
\newblock {\em SIAM Journal on Optimization}, 32(2):1156--1183, 2022.

\bibitem{van1996geometric}
Lou Van~den Dries and Chris Miller.
\newblock Geometric categories and o-minimal structures.
\newblock {\em Duke Mathematical Journal}, 84(2):497--540, 1996.

\bibitem{vazquez2020citylearn}
Jose~R Vazquez-Canteli, Sourav Dey, Gregor Henze, and Zoltan Nagy.
\newblock Citylearn: Standardizing research in multi-agent reinforcement learning for demand response and urban energy management.
\newblock {\em arXiv preprint arXiv:2012.10504}, 2020.

\bibitem{xie2021bellman}
Tengyang Xie, Ching-An Cheng, Nan Jiang, Paul Mineiro, and Alekh Agarwal.
\newblock Bellman-consistent pessimism for offline reinforcement learning.
\newblock {\em NeurIPS}, 34:6683--6694, 2021.

\bibitem{yu2021combo}
Tianhe Yu, Aviral Kumar, Rafael Rafailov, Aravind Rajeswaran, Sergey Levine, and Chelsea Finn.
\newblock Combo: Conservative offline model-based policy optimization.
\newblock {\em NeurIPS}, 34:28954--28967, 2021.

\end{thebibliography}

\appendices
\section{Additional theoretical results and proofs}

\textbf{Proof of Lemma \ref{lem:CompGenFuncApprox}}.

\begin{proof}
    Firstly, we obtain the bound for general function approximation. Due to Lipschitzness assumption of the policy, we have:
    \begin{equation*}
        \|\pi_{\theta^*}(s) - \pi_{\theta}^{gen}(s)\|_2 \leq L_\pi \|\theta^* - \theta^{gen}\|_2 \qquad \forall s \in \mathcal{S},
    \end{equation*}
    where $L_\pi$ is the Lipschitz constant. If the magnitude of each entry in the mismatch vector $(\theta^* - \theta^{gen})$ is upper bounded by constant $Z>0$, then
\begin{equation*}
\begin{aligned}
    \|\theta^* - \theta^{gen}\|_2 \leq &\sqrt{\underbrace{Z^2+ \cdots + Z^2}_{d(T+1) \hspace{0.2cm} \text{times}} } \\ = & \quad Z\sqrt{d(T+1)}.
    \end{aligned}
\end{equation*}
Next, we obtain the upper bound for $\theta^{iAC}$. Using the EDS property and equation \ref{eq:EDSsampleEqn}, we can write
\begin{equation*}
    \begin{aligned}
        &\|\pi_{\theta^*}(s) - \pi_{\theta^{iAC}}(s)\|_2 \leq \mathcal{O}\left(\sum_{t=0}^T\lambda^t H \|\zeta_t^* - \zeta_t^{iAC}\|_2\right) \\&\leq  \mathcal{O}\Bigg(\lambda\underbrace{ \sqrt{\underbrace{Z^2+ \cdots + Z^2}_{d \hspace{0.2cm} \text{times}}} + \cdots +\lambda^{T} \sqrt{\underbrace{Z^2+ \cdots + Z^2}_{d \hspace{0.2cm} \text{times}} } }_{T+1 \hspace{0.2cm} \text{times}} \Bigg)  \\& = \mathcal{O}\left(\lambda Z\sqrt{d} +\lambda^2 Z\sqrt{d}+ \cdots + \lambda^{T} Z \sqrt{d} \right) \\ &= \mathcal{O}\left(Z\sqrt{d} \frac{(1-\lambda^{T+1})}{1-\lambda}\right).
    \end{aligned}
\end{equation*}  
This completes the proof.
\end{proof}

\textbf{Monotonicity property of the solution functions.}  We introduce the monotonicity property of the optimization solution functions, which states that the output of a monotone function applied to the optimal value of an optimization is equal to the optimal value of the optimization where the same monotone function is applied to the objective function. The next lemma proves the statement formally.

\begin{lemma}[Monotonicity lemma]\label{lem:MonotonicityLemma} Let $P_1 \coloneqq 
\{\underset{a}{\max}\quad\bar{r}(a) \quad \text{s.t.} \quad G(a) \leq 0\}$ and $P_2 \coloneqq 
\{\underset{a}{\max}\quad\psi\left(\bar{r}(a)\right) \quad \text{s.t.} \quad G(a) \leq 0\}$ denote two convex optimization problems, where $\psi:\mathbb{R}\rightarrow\mathbb{R}$ is some monotone function. Let $a^*$ and $a_\psi^*$ denote the solutions of $P_1$ and $P_2$, respectively. Then both the solutions are same, i.e., $a^* = a_\psi^*$.
\end{lemma}
    
\begin{proof}
We denote the feasible region $G(a)\leq0$ as $\mathcal{A}$. As $a^{\star}$ is the solution for $P_1$, we can write:
\begin{equation}\label{eq:Mono0}
    \bar{r}(a^{\star}) \geq \bar{r}(a) \hspace{0.3cm} \forall a \in \mathcal{A}.
\end{equation}
Similarly, 
\begin{equation}\label{eq:Mono1}
    \psi\big(\bar{r}(a_\psi^*)\big) \geq \psi\big(\bar{r}(a)\big) \hspace{0.3cm} \forall a \in \mathcal{A}.
\end{equation}
Furthermore, we know for any monotone function, if $x > y$, then $\psi(x) > \psi(y)$. Therefore, using \eqref{eq:Mono1}, we can write
\begin{equation}\label{eq:Mono2}
    \bar{r}(a^*_\psi) \geq \bar{r}(a) \hspace{0.3cm} \forall a \in \mathcal{A}.
\end{equation}
From \eqref{eq:Mono0} and $\eqref{eq:Mono2}$, we can write $a^* = a_\psi^*$. Therefore, the optimal solution stays invariant under monotone operations on the objective function of the optimization problem.
\end{proof}

The main implication of the above lemma is that a monotone function over the optimal value of an optimization ($\psi(\bar{r}(a^*))$ in the above lemma) is the same as the optimal value from an optimization with a monotone applied over its objective function $\psi(\bar{r}(a_\psi^*))$. 


\subsection{Realizability and statistical complexity for optimization-based actor and critic}
In this section, we provide upper bounds on joint statistical complexity $S_{\mathcal{F},\Pi}$ and realizability error $\epsilon_{\mathcal{F}}$ for cases when the actor optimization can be represented by a multi-stage mp-LP/QP. 

\textbf{Realizability error.} First, we establish the upper bound on realizability error $\epsilon_{\mathcal{F}}$ (see Assumption \ref{asmptn:realizability}). 
Intuitively, $\epsilon_{\mathcal{F}}$ can be understood as the maximum approximation error that the critic function class $f_\omega^\theta \in \mathcal{F}$ will have to approximate some $Q^{\pi_\theta}$. It is shown in \cite[Theorem 1]{jin2023solution} that the maximum approximation error of an optimization solution function to approximate a twice-differentiable smooth function is $\epsilon$ if the mp-LP has $\mathcal{O}\left(\left(\frac{n_s}{\epsilon} \right)^{n_s/2} \right)$ inequality constraints. Therefore, we can obtain the following approximation error if $n_{\mathcal{I}}$ inequality constraints are used:
\begin{equation*}
    \epsilon = \mathcal{O}\left( \frac{n_s}{n_{\mathcal{I}}^{2/n_s}} \right).
\end{equation*}

The above bound will also hold for $\epsilon_{\mathcal{F}}$ due to the Lipschitzness assumption of the critic function class $\mathcal{F}$ with respect to the actions. 

\begin{lemma}\label{lem:epsilon_F}
The following holds if the actor optimization can be represented by a multistage mp-LP/QP with $n_{\mathcal{I}}$ inequality constraints:
    \begin{equation*}
        \epsilon_{\mathcal{F}} \leq \mathcal{O}\left( \frac{n_s}{(n_{\mathcal{I}})^{2/n_s}}\right).
    \end{equation*}
\end{lemma}
    
We can observe that as we keep increasing the number of constraints $n_{\mathcal{I}}$, the approximate realizability error will diminish.

\textbf{Statistical complexity.} In this work, we use the $\mathcal{L}_1$ covering number bounds to characterize the statistical complexity for the actor and critic function class. We obtain the covering number bounds for both the actor and critic separately.

\textbf{Covering number bound for the actor.} In the case when the actor optimization can be expressed using  multi-stage mp-LP or mp-QP with $n_{\mathcal{I}}$ linear inequality constraints and $n_{\mathcal{E}}$ equality constraints, we can obtain the covering number bound for the function class $\Pi$ using the result from \cite{jin2023solution} as follows:

\begin{lemma}\label{lem:PiCoveringBound} 
The log-covering number bound for the function class $\Pi$, when $\Pi$ comes from an mp-LP/QP, can be bounded as follows:
\begin{equation*}
\small
\begin{aligned}
    \label{lem:CoveringNumberActor}
    \log \mathcal{N}_1(\Pi,\epsilon) \leq \mathcal{O}\left(\frac{\kappa^{*2}}{\epsilon^2} {\sum}_{0\leq i\leq {(T+1)(n_a-n_{\mathcal{E})}}}\begin{pmatrix}(T+1)n_{\mathcal{I}}\\ i\end{pmatrix}\right),
    \end{aligned}
\end{equation*}
where $\kappa^*$ is some condition number associated with the optimization parameters,
and $\begin{pmatrix}a\\b\end{pmatrix}$ denotes $a$ choose $b$.
\end{lemma}
\begin{proof}
    Firstly, from \cite{jin2023solution}, the solution function is piecewise affine. Each piece can be bounded with $\kappa^{*2}/\epsilon^2$, where $\kappa^*$ is a condition number associated with the reformulated mp-LP or QP \cite{jin2023solution}. Second, the maximum number of critical regions can be bounded by ${\sum}_{0\leq i\leq {(T+1)(n_a-n_{\mathcal{E})}}}\begin{pmatrix}(T+1)n_{\mathcal{I}}\\ i\end{pmatrix}$ \cite{jin2023solution}.
\end{proof}

\textbf{Covering number bound for the critic.}
Next, we present the covering bound for the critic function class $\mathcal{F}$. Note that the function class $\mathcal{F} \coloneqq \{\psi\big(\phi(\cdot,\pi_\theta(\cdot);\cdot)\big): \pi_\theta \in \Pi, \phi \in \Phi^{opt}, \psi \in \mathcal{F}^{mono}, \theta \in \Theta \}$, where $\Phi^{opt}$ is the optimal value function class denoted as $\Phi^{opt} \coloneqq \{f^{L}\big(\cdot,\pi_\theta(\cdot);\cdot\big), \pi_\theta \in \Pi,f^{L}\in \mathcal{F}^{L}  \}$, $\mathcal{F}^{L}$ is the linear function class, as we consider the representation of actor through an mp-LP. Function class $\mathcal{F}^{mono}$ denotes the monotonic function class. This makes critic function class $\mathcal{F}$ a composite function of other functions. Therefore, we first present the covering number of other constituent function classes and then provide the overall covering number bound for the function class $\mathcal{F}$. 

First, the covering number for the function class $\mathcal{F}^{L}$ is bounded using the covering number for the linear function class \cite[Corollary 9]{kakade2008complexity} given as $\log \mathcal{N}_1(f^L,\epsilon) \leq \mathcal{O} \left(1/\epsilon^2\right)$. Second, we provide the covering number bound for the uniformly bounded monotonic function class, which follows directly from \cite[Theorem 2.7.5]{van1996geometric}.

\begin{lemma}
\label{lem:MonotoneCovering}
The covering number bound for the class $\mathcal{F}^{mono}$ of bounded monotone functions is given as:
\begin{equation*}
    \log \mathcal{N}_1(\epsilon, \mathcal{F}^{mono}) \leq C_{mono}\bigg(\frac{1}{\epsilon} \bigg),
\end{equation*}
where $C_{mono}$ is a positive constant.
\end{lemma}

Next, we provide the covering number bound for the optimal value function class $\Phi^{opt}$.

\begin{lemma}\label{lem:OptimalValueCovering}
Consider the function class $\Phi^{opt} \coloneqq \{f^{L}\big(\cdot,\pi_\theta(\cdot);\cdot\big), \pi_\theta \in \Pi,f^{L}\in \mathcal{F}^{L}  \}$, where $\mathcal{F}^{L}$ is the linear function class. Then for any constant $c_1$, $c_2 \in (0,1)$, the covering number for $\Phi^{opt}$ is bounded as
\begin{align*}
    \log \mathcal{N}_1(\epsilon, \Phi^{opt}) \leq \mathcal{O}\left( \frac{C_{\mathcal{F},\Pi}}{\epsilon^2}\right),
\end{align*}
where $C_{\mathcal{F},\Pi} = {\sum}_{0\leq i\leq {(T+1)(n_a-n_{\mathcal{E})}}}\begin{pmatrix}(T+1)n_{\mathcal{I}}\\ i\end{pmatrix}$, and constants $c_1$, $c_2$ are hidden inside the big O notation.
\end{lemma}
\begin{proof}
We form the $\epsilon$-coverings around the domain of interest. First, we create a $c_1\epsilon$-cover $\Pi_{j,\epsilon} = \{\pi_{j,1}, ..., \pi_{j,p_j} \}$ and a $(1 - c_1)\epsilon$-cover $\mathcal{F}^{L}_{j,\epsilon} = \{f^{L}_{j,1}, ... , f^{L}_{j,p_{j'}}\}$ over $\mathcal{F}^{L}$. $p_j$ and $p_{j'}$ can be chosen to satisfy $p_j \leq \mathcal{N}_1\left(c_1 \epsilon,\Pi \right),$ and $p_{j'} \leq \mathcal{N}_1\big((1-c_1)\epsilon, \mathcal{F}^{L} \big)$. Consider the set of functions
\begin{equation*}
    \Phi^{opt}_{\epsilon} \coloneqq \left\{ f^{L}_{j,i_j'}(\cdot,\pi_{j,i_j'}(\cdot);\cdot), \forall i_j \in [p_j], \forall i_{j'} \in [p_{j'}], \pi \in \Pi \right\},
\end{equation*}

where each member function forms the $c_1\epsilon$-set and $(1-c_1)\epsilon$-set from every region in $\mathcal{S}$ and $\Pi$, respectively. Since, $\Pi_{j,\epsilon}$ and $\mathcal{F}^{L}_{j,\epsilon}$ are $c_1\epsilon$ and $(1-c_1)\epsilon$ covers of $\mathcal{S}$ and $\Pi$, there exists a selection function $i_j(\pi)$ and $i_{j'}(\psi)$ such that for any $s \in \mathcal{S}$ and $a \in \mathcal{A}$, the following holds $$|\pi(s) - \pi_{j,i_j}(s)| \leq \epsilon,$$ and $$|f^{L}(\cdot,\pi(s);\cdot) - f^{L}_{j,i_{j'}}(\cdot,\pi_{j,i_j}(s);\cdot)| \leq \epsilon.$$
Thus, $\Phi^{opt}_{\epsilon}$ forms an $\epsilon$-net of $\Pi.$ The result can be obtained by taking $|\Phi^{opt}_\epsilon|$ which is bounded as $$|\Phi^{opt}_\epsilon|\leq \mathcal{N}_1\left(c_1 \epsilon,\mathcal{F}^L \right)\mathcal{N}_1\big((1-c_1)\epsilon, \Pi \big).$$
As $\mathcal{F}^L$ is a class of linear function class, we bound $\mathcal{N}_1(\epsilon,\mathcal{F}^L)$ by $\mathcal{O}\left(\exp(1/\epsilon^2)\right)$ \cite[Corollary 9]{kakade2008complexity}. Similarly, by substituting the covering number bound for $\Pi$ from Lemma \ref{lem:PiCoveringBound}, we can obtain the final result.
\end{proof}

Now we finally prove the covering number bound for the critic function class $\mathcal{F}$.
\begin{lemma}
\label{lem:appRWFcoveringBound}
Consider the function class $\mathcal{F} \coloneqq \{\psi\big(\phi(\cdot,\pi_\theta(\cdot);\zeta);\omega \big), \pi_\theta \in \Pi, \phi \in \Phi_{opt}, \psi \in \mathcal{F}^{mono}, \theta \in \Theta \}$. Then for any constant $c_1,c_2 \in (0,1)$, the following  holds:
\begin{equation*}
    \log \mathcal{N}_1(\mathcal{F},\epsilon) \leq \mathcal{O}\left(\frac{(n_a - n_{\mathcal{E}}) 2^{n_{\mathcal{I}}} + \epsilon \sqrt{n_{\mathcal{I}}}}{\epsilon^2 \sqrt{n_{\mathcal{I}}}} \right).
\end{equation*}
\end{lemma}

\begin{proof}
    We form the $\epsilon$-coverings around the domain of interest. First, create a $c_1\epsilon$-cover $\Pi_{j,\epsilon} = \{\pi_{j,1}, ..., \pi_{j,p_j} \}$ and a $(1 - c_1)\epsilon$-cover $\mathcal{F}^{{mono}}_{j,\epsilon} = \{\psi_{j,1}, ... , \psi_{j,p_{j'}}\}$ over $\Pi$. $p_j$ and $p_{j'}$ can be chosen to satisfy $$p_j \leq \mathcal{N}_1\left(c_1 \epsilon,\Pi \right),$$ and $$p_{j'} \leq \mathcal{N}_1\big((1-c_1)\epsilon, \mathcal{F}^{{mono}} \big).$$

Consider the set of functions
\begin{equation*}
\begin{aligned}
    \mathcal{F}_{\epsilon} \coloneqq \{ & Q_{\psi}: Q_{\psi} = \psi_{j,i_j'}(\pi_{j,i_j};\omega), \forall i_j \in [p_j], \\ & \forall i_{j'} \in [p_{j'}], \omega \in \Omega \},
    \end{aligned}
\end{equation*}
where each member function forms the $c_1\epsilon$-set and $(1-c_1)\epsilon$-set from every region in $\mathcal{S}$ and $\Pi$ respectively.. Since, $\Pi_{j,\epsilon}$ and $\mathcal{F}^{mono}_{j.\epsilon}$ are $c_1\epsilon$ and $(1-c_1)\epsilon$ covers of $\mathcal{S}$ and $\Pi$, there exists a selection function $i_j(\pi)$ and $i_{j'}(\psi)$ such that for any $s \in \mathcal{S}$ and $a \in \mathcal{A},$ the following holds $$|\pi(s) - \pi_{j,i_j}(s)| \leq \epsilon,$$ and $$|\psi(\pi;\omega) - \psi_{j,i_{j'}}(\pi_{j,i_j'};\omega)| \leq \epsilon.$$
Thus, $\mathcal{F}_{\epsilon}$ forms an $\epsilon$-net of $\Pi$. The result be obtained by taking $|\mathcal{F}_\epsilon|$ which is bounded as $$|\mathcal{F}_\epsilon|\leq \mathcal{N}_1\left(c_1 \epsilon,\Pi \right)\mathcal{N}_1\big((1-c_1)\epsilon, \mathcal{F}^{mono} \big).$$

Therefore, we can write
\begin{equation*}
\begin{aligned}
    \mathcal{N}_1(\epsilon, \mathcal{F}) & \leq \mathcal{N}_1\left(c_1 \epsilon,\Phi_{opt} \right) \mathcal{N}_1\big((1-c_1)\epsilon, \mathcal{F}^{mono} \big) \\  & = \exp \left(\frac{C_{\mathcal{F},\Pi}}{\epsilon^2} +\frac{C_{mono}}{\epsilon}\right)\\&= \exp \left(\frac{C_{\mathcal{F},\Pi}+ C_{mono}\epsilon}{\epsilon^2 } \right).
    \end{aligned}
\end{equation*}
This completes the proof.
\end{proof}

The next lemma provides the bound on the joint statistical complexity $S_{\mathcal{F},\Pi}$ for the specific case when the actor optimization can be represented by an mp-LP.
\begin{lemma}
    \label{lem:epsilonStatBound}
    The following bound holds
    \begin{equation*}
        S_{\mathcal{F},\Pi} \leq \mathcal{O}\left(\frac{\left[ 2C_{\mathcal{F},\Pi}+C_{mono}\epsilon\right]}{\delta \epsilon^2} \right).
    \end{equation*}
\end{lemma}
\begin{proof}
    We know $S_{\mathcal{F},\Pi}\coloneqq \frac{\log \left(\mathcal{N}_1(\mathcal{F}, \epsilon)\mathcal{N}_1(\Pi, \epsilon)\right)}{\delta}$. Substituting the upper bound for both the covering numbers for $\Pi$ and $\mathcal{F}$ from Lemma \ref{lem:PiCoveringBound} and \ref{lem:appRWFcoveringBound}, repsectively, we can obtain the final result.
\end{proof}

\textbf{Proof of Theorem \ref{thm:iACguarantee}.}

\begin{proof}
The performance difference can be decomposed as follows \cite[Lemma 12]{cheng2022adversarially}:
\begin{equation*}
\small 
\begin{aligned}
    J(\pi_{\theta^*}) - J(\pi_{\theta^{(K)}}) = & \frac{1}{1-\gamma}\Bigg[\underbrace{\mathbb{E}_\mu\left[\left(f_K - \mathcal{T}^{\pi_{\theta^{(K)}}}f_K\right)(s,a)\right]}_{\text{(A)}}  \\ &+\underbrace{\mathbb{E}_{\pi_{\theta^*}}\left[\left(\mathcal{T}^{\pi_{\theta^{(K)}}}f_K - f_K\right)(s,a)\right]}_\text{(B)} \\&+\underbrace{\mathbb{E}_{d^{\pi_{\theta^*}}}\left[f_K(s,\pi_{\theta^*}(s)) - f_K(s,\pi_{\theta^{(K)}}(s))\right]}_\text{(C)} \\ &+ \underbrace{\mathcal{L}_\mu(\theta^{(K)},\omega^{(K)}) -\mathcal{L}_\mu(\theta^{(K)},Q^{\pi_{\theta^{(K)}}})}_\text{(D)}\Bigg],
    \end{aligned}
\end{equation*}
where we denote $f_{\omega_K}^{\theta^{(K)}}$ as $f_K$ for notational simplicity. The first two term $(A)$ and $(B)$ are the average Bellman errors for critic $f$ and policy $\pi_{\zeta^{(k)}}$ with respect to the offline data distribution $\mu$ and comprator policy distribution $\pi_{\zeta^*}$. We bound the terms (A), (B), and (D) as done in \cite[Theorem 14]{cheng2022adversarially} as follows:
\begin{equation*}
    (\text{A}) \leq \frac{\sqrt{\epsilon_b}+\sqrt{\psi_{max}/\beta}}{1-\gamma},
\end{equation*}
\begin{equation*}
\small
\begin{aligned}
    (\text{B}) \leq \frac{2\sqrt{C^*}\left(\sqrt{\epsilon_b}+\sqrt{\psi_{max}/\beta}\right)}{1-\gamma}+ \frac{<d^{\pi_{\zeta^*}}\backslash \nu,f_k - \mathcal{T}^{\pi_{\zeta^{(k)}}}f_k>}{1-\gamma},
    \end{aligned}
\end{equation*}
\begin{equation*}
    (\text{D}) \leq \sqrt{\epsilon_{\mathcal{F}}}+\sqrt{\epsilon_{stat}} + \beta.\mathcal{O}\left(\epsilon_{\mathcal{F}} + \epsilon_{stat}\right),
\end{equation*}
where $\epsilon_{stat} = \mathcal{O}\left(\frac{\psi_{max}^2 S_{\mathcal{F},\Pi}}{n} \right)$, $C^*$ is the constant defined in Assumption \ref{asmptn:boundedC}, $\epsilon_b = \sqrt{\mathcal{E}_{\mathcal{D}}(f,\pi)}+\mathcal{O}\left(\psi_{max}\sqrt{\epsilon_{stat}/n}\right)$, where $\mathcal{E}_{\mathcal{D}}(f,\pi)$ denotes the Bellman residual error for critic $f$ and policy $\pi$. 

The upper bound on (C) is obtained as follows using the Lemma \ref{lem:CompGenFuncApprox}:
\begin{equation*}
    \begin{aligned}
        (C) & = \mathbb{E}_{d^{\pi_{\theta^*}}}\left[f_k(s,\pi_{\theta^*}) - f_k(s,\pi_{\theta^{(k)}}) \right] \\ & \leq L_\psi \mathbb{E}_{d^{\pi_{\theta^*}}}\left[ \|\pi_{\theta^*}(s) - \pi_{\theta^{(k)}}\|_2\right] \\ & \leq \mathcal{O}\left(Z\sqrt{d}\frac{(1-\lambda^{T+1})}{1-\lambda} \right),
    \end{aligned}
\end{equation*}
where $Z$ is the positive constant from Lemma \ref{lem:CompGenFuncApprox}. Combining all the upper bounds after substituting $\beta = \mathcal{O}\left(\frac{\psi_{max}^{1/3}n^{2/3}\delta^{2/3}}{C_{\mathcal{F},\Pi}^{2/3}} \right)$, we get the final result. The result for the general function approximation can be obtained in the same way using the upper bound on $\|\pi_{\theta^*}(s) - \pi_{\theta^{gen}}(s)\|_2$ from Lemma \ref{lem:CompGenFuncApprox}, to see the benefit of the EDS property.
\end{proof}

\section{Additional details for experiments}

\subsection{Citylearn challenge}
\textbf{Environment details.} We refer the reader to \cite{vazquez2020citylearn} and the corresponding online documentation\footnote{link: \url{https://sites.google.com/view/citylearnchallenge}} for the detailed setup of the competition. To build a lookahead model for the actor, we learn a set of predictors for solar generation and electricity/thermal demands. Prediction is done on a rolling-horizon basis for the next $24$ hours using the past 2 weeks data. Denote the hour index by $r \in \{1,2,... ,T \}$, where $T = 24$. Suppose that we are at the beginning of hour $r$. Then we need to plan for the action for the upcoming hour $r$ (note that we need to plan for future hours in the process). Further details on the environment constraints and decision variables can be found in \cite[Appendix B]{khattar2023winning}.

\textbf{Experimental setup details:} For the experiments on the CityLearn challenge, we considered the data from the Climate Zone 1 provided by the CityLearn environment. To introduce some kind of stationarity in the MDP, we take 1 month of data and recreate this 1-year data by replicating the 1-month data over 12 months. The RBC policy provided by the environment is used to collect the trajectory dataset. Therefore, in this case, the length of the trajectory dataset is $n = 2190$, which is equivalent to 24 hours over 90 days.

\textbf{RWF design.} We design the monotonic reward warping function for the critic as follows:
\begin{align*}
    \psi(\phi;\omega) = \omega_1\phi+ \omega_2,
\end{align*}
where $\phi$ is the optimal value from the actor optimization. The parameters $\omega_1$ and $\omega_2$ were learned using the least squares regression.

\subsection{Details for the iAC implementation}

ADAM optimizer is used for the updates of iAC. A total number of 2190 samples were used to train the actor optimization parameters. The actor learning rate $\alpha_a = 0.015$ is used. To achieve stable learning, we use target actors and target critics in our implementation where the target actor parameter $\zeta$ is updated as follows:
\begin{align*}
    \zeta^{target} \leftarrow \alpha\zeta^{target}+(1-\alpha)\zeta_k,
\end{align*}
where $\alpha$ is the actor update parameter, $\zeta_k$ is the actor update obtained at some iteration $k$; $\alpha$ is set to be 0.25 in our implementation. We also use TD($\lambda$) for our critic evaluations with $\lambda = 0.95$. Note this $\lambda$ is different from the EDS sensitivity parameter $\lambda$.


\subsection{Details for other offline RL baselines}

The other offline RL baselines are also trained on the same trajectory dataset collected by the behavior policy. The standard MLP architecture is used for the policy that consists of two linear layers with 256 hidden units.

\begin{table}[htbp]
\centering
\begin{tabular}{llllll}
\multicolumn{1}{c}{\multirow{2}{*}{\textbf{Parameter}}} & \multicolumn{5}{c}{\textbf{Value}}                          \\  
\multicolumn{1}{c}{}                                            & BEAR  & AWAC   & CQL  & BCQ  &  TD3+BC \\ \hline
actor learning rate                                           & \multicolumn{1}{c}{1e-4} & 3e-4 & 3e-4 & 1e-4 & 3e-4 \\
critic learning rate                                           & \multicolumn{1}{c}{3e-4} & 3e-4 & 3e-4 & 1e-3 & 3e-4 \\
imitator learning rate                                            & 3e-4                       & NA     & NA   & 1e-4    & NA   \\
batch size                                             & 256                      & 1024    & 100  & 100   & 256  \\
number of critics                                              & 2                     & 2    & 1  & 2   & 2  \\
$\lambda$ for TD($\lambda$)                                              & 0.75                     & 1.0    & NA  & 0.75   & NA  \\
discount factor                                         & 0.99                     & 0.99 & 0.99 & 0.99 & 0.99
\end{tabular} \caption{Parameter values used for the other offline RL baselines. NA is not applicable.}
\end{table}

\subsection{Supply chain management}

Here, present the extra details for the supply chain management problem. System dynamics are given by $$h_{t+1} = h_t + (A^{\text{in}} - A^{\text{out}})a_t,$$ where $$A^{\text{in}} \in \mathbb{R}^{4 \times 8}$$ and $$A^{\text{out}} \in \mathbb{R}^{4 \times 8}$$ are known matrices, and actions $a_t = (b_t, s_t, z_t)$.

The following constraints are imposed as part of the supply chain management:
\begin{subequations}
\begin{align}
    0\leq h_t \leq h_{\text{max}}\\
   0\leq b_t \leq b_{\text{max}}\\
   0\leq g_t \leq g_{\text{max}}\\
   0\leq z_t \leq z_{\text{max}}\\
    A^{\text{out}}(b_t, s_t, z_t) \leq h_t\\
    g_t \leq d_t
\end{align}
\end{subequations}
The uncertainty in future demand $d_t^c$ and supplier prices $p_t$ are modeled using the log-normal distribution:
\begin{align}
\log w_t \coloneqq (\log p_{t+1}, \log d_{t+1}^c) \sim \mathcal{N}(\mu, \Sigma).
\end{align}

The true cost to minimize is given by
\begin{equation}
    \frac{1}{T}\sum_{t=0}^{T-1} p_t^Tb_t - r^Ts_t + \tau^Tz_t + \alpha^Th_t + \beta^Th_t^2.
    \label{eq:app_true_cost_supply_chain}
\end{equation}
\textbf{Optimization planner:} The surrogate convex optimization model is parametrized by $\zeta = (P,q)$ for some $P \in \mathbb{R}^{4 \times 4}$ and $q \in \mathbb{R}^4$:
\begin{equation}
\begin{aligned}
\underset{a_t=(b_t,s_t,z_t)}{\min} \quad & p^T_tb_t - r^Ts_t + \tau^Tz_t - ||Ph_t||^2_2 - q^Th_t\\
\textrm{s.t.} \quad & h_{\text{t+1}} = h_t + (A^{\text{in}} - A^{\text{out}})a_t\\
  &0 \leq h_t \leq h_{\text{max}}\\
  &0 \leq b_t \leq b_{\text{max}}\\
  &0 \leq s_t \leq s_{\text{max}}\\
  &0 \leq z_t \leq z_{\text{max}} \\
\end{aligned}
\end{equation}

The true cost to minimize can be accessed directly without additional noise and is given by:
\begin{equation}
    \frac{1}{T}\sum_{t=0}^{T-1} p_t^Tb_t - r^Ts_t + \tau^Tz_t + \alpha^Th_t + \beta^Th_t^2
    \label{eq:true_cost_supply_chain}
\end{equation}

Here $P$ and $q$ are meant to capture the effect of randomness from  $\alpha$ and $\beta$ that appear inside the original cost function. The RL agent is tasked learn the best parameters resulting in the maximum profit. In the following, we consider a setup with $n = 4$ nodes, $m = 8$ links, $k = 2$ supply links, and $c = 2$ customer links.

\end{document}